\documentclass[final]{cvpr}

\usepackage{times}
\usepackage{multirow}
\usepackage{epsfig}
\usepackage{graphicx}
\usepackage{amsmath}
\usepackage{amsthm}
\usepackage{amssymb}

\usepackage[pagebackref=true,breaklinks=true,colorlinks,bookmarks=false]{hyperref}

\def\cvprPaperID{8739}
\def\confYear{CVPR 2021}

\usepackage{graphicx}
\usepackage{caption}
\usepackage{xcolor}

\usepackage[utf8]{inputenc} 
\usepackage[T1]{fontenc}   
\usepackage{hyperref}      
\usepackage{url} 
\usepackage{booktabs}       
\usepackage{amsfonts}       
\usepackage{nicefrac}       
\usepackage{microtype}      

\usepackage{color,soul}
\usepackage{bbm}
\usepackage{mathtools}

\newtheorem{theorem}{Theorem}

\usepackage{algorithmic}
\usepackage[ruled,vlined]{algorithm2e}

\DeclarePairedDelimiterX{\norm}[1]{\lVert}{\rVert}{#1}

\begin{document}

\title{Self-Gradient Networks}

\author{
  \normalfont{Hossein Aboutalebi\textsuperscript{1}, Mohammad Javad Shafiee\textsuperscript{1}}\\
 Alexander Wong\textsuperscript{1}\\
 \textsuperscript{1}Waterloo AI Institute, University of Waterloo, Waterloo, Ontario, Canada\\
  \textsuperscript{1}\{haboutal, mjshafiee, a28wong\}@uwaterloo.ca \\
}

\maketitle

\begin{abstract}
   The incredible effectiveness of adversarial attacks on fooling deep neural networks poses a tremendous hurdle in the widespread adoption of deep learning in safety and security-critical domains.  While adversarial defense mechanisms have been proposed since the discovery of the adversarial vulnerability issue of deep neural networks, there is a long path to fully understand and address this issue.
   In this study, we hypothesize that part of the reason for the incredible effectiveness of adversarial attacks is their ability to implicitly tap into and exploit the gradient flow of a deep neural network.  This innate ability to exploit gradient flow makes defending against such attacks quite challenging.  Motivated by this hypothesis  we argue that if a deep neural network architecture can explicitly tap into its own gradient flow during the training, it can boost its defense capability significantly.   
   Inspired by this fact, we introduce the concept of \textbf{self-gradient networks}, a novel deep neural network architecture designed to be more robust against adversarial perturbations.  Gradient flow information is leveraged within self-gradient networks to achieve greater perturbation stability beyond what can be achieved in the standard training process.  We conduct a theoretical analysis to gain better insights into the behaviour of the proposed self-gradient networks to illustrate the efficacy of leverage this additional gradient flow information.  The proposed self-gradient network architecture enables much more efficient and effective adversarial training, leading to faster convergence towards an adversarially robust solution by at least 10$\times$.  Experimental results demonstrate the effectiveness of self-gradient networks when compared with state-of-the-art adversarial learning strategies, with $10\%$ improvement on the CIFAR10 dataset under PGD and CW adversarial perturbations.

\end{abstract}

\section{Introduction}

Significant advances in deep learning in recent years has led to tremendous interest and impact on industry, with many considering deep learning as a disruptive technology.  More specifically, recent breakthroughs in deep learning-driven computer vision has enabled different industries to leverage deep learning for a wide range of complex tasks such as object detection~\cite{girshick2015fast, bolya2019yolact}, image segmentation~\cite{ronneberger2015u, he2017mask} and image classification~\cite{he2016deep}. Yet, despite these successes, deep learning remains limited in its widespread adoption in safety and security-critical tasks.

One of the major concerns that has acted as a hurdle to widespread adoption of deep learning in safety and security critical domains has been the vulnerability of deep neural networks to adversarial attacks~\cite{apruzzese2019addressing, fawzi2014secure, melis2017deep}. It was first discovered  by~\mbox{Szegedy {\it et al.}} in their seminal paper~\cite{szegedy2013intriguing} that if a perturbation $\epsilon$ is added in a specific direction to the input of a deep neural network, it can cause a network to misclassify an input to an attacker-desired class with a high confidence. Such adversarial attacks are designed such that the perturbation $\epsilon$ when added to the input sample is indistinguishable by the human eye.  In order to achieve the goal of being human indistinguishable, a restriction is imposed on the norm of the perturbation. Such a perturbed input is known as an adversarial example, and poses danger to the use of deep learning in safety and security-critical domains since such attacks cannot be visually identified via human-in-the-loop inspection.

It has been shown that the adversarial vulnerability of deep neural networks is not unique to computer vision tasks and can be extended to other tasks such as speech recognition~\cite{alzantot2018did, qin2019imperceptible}, reinforcement learning~\cite{pattanaik2017robust}, and text classification~\cite{liang2017deep}.

The study started by~\mbox{Szegedy {\it et al.}}~\cite{szegedy2013intriguing} led to significant breakthroughs in understanding and inventing more advanced adversarial attacks to fool deep neural networks~\cite{moosavi2016deepfool, goodfellow2014explaining, madry2017towards,carlini2017towards}. In this regard, fast gradient sign method (FGSM) attack introduced by Goodfellow {\it et al.}~\cite{goodfellow2014explaining} is one of the fastest methods for generating adversarial examples. The sign of the gradient for the loss function is used as a perturbation direction to be added to the input sample.  Moosavi {\it et al.} proposed DeepFool~\cite{moosavi2016deepfool}, one of the first multi-step attacks that can create adversarial examples which are harder to classify.  One of the novelties of this work lies in the fact that this attack is designed such that the magnitude of perturbation in the input is minimized. Madry {\it et al.}~\cite{madry2017towards} proposed a multi-step attack so-called projected gradient descent (PGD) applying FGSM attack iteratively followed by a projection.  They showed that PGD is the hardest first-order adversarial attack in research literature and many of the prior first-order attacks can be seen as a special case of this particular attack.

As the field of adversarial machine learning continues to evolve and advance with more and more advanced attacks, the field of adversarial defense against these attacks has emerged and has gained great attention as well. 
 Among the many defense mechanisms proposed so far, adversarial training~\cite{goodfellow2014explaining} remains one of the core algorithms used to improve adversarial robustness of deep neural networks. Adversarial training was first proposed by  Goodfellow {\it et al.}~\cite{goodfellow2014explaining}. It consists of augmenting the training set with adversarial examples to help the model generalize outside its training domain. This technique is further examined~\cite{kurakin2016adversarial,szegedy2016rethinking} on larger datasets and different neural network architectures. In this regard, different studies~\cite{sinha2017certifying, raghunathan2018certified}   have explored and provided statistical guarantees on the robustness of their proposed adversarial training mechanisms. However, the defensive distillation strategy proposed by Papernot {\it et al.}~\cite{papernot2016distillation} takes a different perspective to adversarial defense by changing the training data labels such that the Softmax outputs of a deep neural network are replaced with a smoothed values. Although this approach was successful in increasing the network robustness against some of the early adversarial attacks, more advanced multi-step attacks like CW~\cite{carlini2017towards} were still able to fool the network. 
 
 Changing the architecture of the network is another solution to mitigate the effect of  adversarial attacks and improve the robustness of the network. Xie {\it et al.}~\cite{xie2019feature} proposed a novel denoising block  which essentially denoises the feature map in the deep neural networks. The proposed block is used as a new operation block after some of the convolutional layers. They take  advantage of multiple denoising operations  inside the denoising block including non-local means~\cite{buades2005non}. Meng \textit{et al.}~\cite{meng2017magnet} proposed a new network architecture called MagNet which consists of modules of detector network for detecting adversarial examples and reformer networks which pushes the perturbed input toward the manifold of normal inputs. Liao \textit{et al.}~\cite{liao2018defense} proposed  a modified version of the denoising autoencoder~\cite{vincent2008extracting} with the U-net architecture~\cite{ronneberger2015u} in the architecture of the deep neural network. This unit is then trained on the error in the top layers of the neural network as loss in order to denoise the image. 

 Since the inception of the field adversarial attack, multiple conjectures have been proposed to justify the mysterious misbehaviour in deep neural network against adversarial attacks. One of the main conjecture is the one proposed by Goodfellow {\it et al.}~\cite{goodfellow2014explaining} which attributes this poor behaviour to the linear behavior of deep neural network in high-dimensional spaces. Here, we further investigate this misbehaviour by considering the possible advantages an adversarial attack can have over the network. We believe that these advantages makes neural networks weak against these kind of attacks.  In this regard, in some adversarial attacks called white-box attacks~\cite{rakin2018parametric}, the attacker has full access to the architecture of the network which provides huge advantage to the attacker by utilizing the gradient of the network. Even in other type of adversarial attacks (i.e., black-box attacks), the attacker has access to a similar architecture which can provide enough information to make the attack successful against the network~\cite{ilyas2018black,papernot2017practical}. As discussed earlier, most of the proposed adversarial attack techniques generate the perturbation by using the gradient of the network which the network does not necessarily learn during the training. This missing information provides a gap in the training process of the network which is directly abused by the adversarial attack. We further explain this phenomenon in the following Sections.
 
 Motivated by this insight, we first provide the theoretical analysis of the problem and explore possible situations that such gradient awareness may exist. We then propose the concept of \textbf{self-gradient networks}, a new deep neural network architecture leverages and learns from the gradient flow to improve its robustness against adversarial attacks. We finally study the effectiveness of the proposed architecture by comparing it with some of the established and state-of-the-art solutions.
 Here, we mainly focus on white-box adversarial attacks which are generally harder than black-box adversarial attacks to solve. As such, the proposed self-gradient networks should be effective against black-box attacks as well.

 \section{Motivation}
 
 One of the main advantage of adversarial attacks over the targeted deep neural network (DNN) is that they have direct access to the architecture of the network. This advantage put adversarial attack one step ahead of the DNN model as they are already aware of the output for a given input. As an extreme example of using this advantage, DeepFool attack~\cite{moosavi2016deepfool} continues to add perturbations in the direction of the gradient  until the DNN starts to misclassify the perturbed input. Although other attacks like PGD or FGSM do not follow a similar iterative scheme and have fixed iteration numbers, this example shows how far the advantage of accessing to the DNN architecture can be abused to the benefit of adversarial attack.

Another key advantage of adversarial attack over the DNN model is the use of gradient of the network to construct the perturbation. During the normal training process, the network is only given the training data as input. In addition, even in the adversarial training process, the perturbed training data without any extra information is fed as input to the DNN model. In both of these cases, it is unlikely for the network to reconstruct the gradient which is heavily used in most adversarial attack to perturb the input. In other words, giving the information regarding the gradient of the DNN model can significantly boost the performance of the model against adversarial attack. 

To validate this hypothesis, we consider a  Wide-ResNet (WRN-28-10, $w=10$)\cite{zagoruyko2016wide} architecture adversarially trained on CIFAR-10~\cite{krizhevsky2009learning} dataset which is trained via Madry PGD-based adversarial training~\cite{madry2017towards} approach. In one case, there is no information regarding the gradient of the network. However, in the other case, the gradient of the loss with respect to the input is provided as an additional input to the network. All other settings in both cases are the similar. Table~\ref{tab:moti} shows the results for PGD attack with 10 steps and clean dataset. As seen, if the model  accesses to the gradient information can outperform the model without accessing to the gradient information significantly by the large margin, close to $40\%$. It is also interesting to see that the gradient also helped the model to have higher accuracy over the clean dataset by the margin of $7\%$. This example clearly illustrates that the benchmark adversarial training is ineffective in reconstructing gradient information during the training time.
While this example reinforces our conjecture regarding the pivotal role of the gradient in improving the results against adversarial attack, it is impractical as we do not have access to the gradient of the loss at test time. However, approximating this information should also improve the performance of the model. 
\begin{table}
\begin{center}
 \begin{tabular}{||c |c |c ||} 
 \hline
\multirow{ 2}{*}{ Model} & \multirow{ 2}{*}{ Clean} & \multirow{ 2}{*}{ PGD10}  \\ [0.5ex]
& & ~Attack \\
 \hline\hline
 With Grad & 92.10 & 82.4  \\ 
 \hline
  Without Grad & 85.7 & 45.1  \\
 \hline
 
\end{tabular}
\caption{Results of the defense models accuracy on CIFAR-10. Both models are trained with Madry PGD-based adversarial training. With Grad refers to the model having gradient data as an extra information. Without Grad is the model only receiving input image. }
	\label{tab:moti}
\end{center}
\end{table}

\subsection{Addressing Challenges}
Although it is impossible to address the first challenge in adversarial attack as it is one of the key assumption of the adversarial attacks (i.e., accessing to the architecture), it is possible to circumvent the second challenge through altering the network architecture. In this work, we propose the concept of self-gradient networks, a deep neural network architecture designed specifically to capture the gradient flow within the network with respect to the input. By explicitly enabling a deep neural network architecture to tap into its own gradient flow, the proposed self-gradient networks can tap into some of the same hidden knowledge exploited by adversarial attack mechanisms to generate adversarial examples. Our experiments show that the self-gradient networks, by harnessing its own gradient flow, can boost robustness significantly in the face of adversarial attacks. 

\section{Self-gradient Networks}
\label{sec:self-gradient}
In this section, the theoretical underpinnings behind the proposed self-gradient networks is described in detail and the insights and intuition behind this concept are provided.  Let us first describe the operation of the fundamental building block of self-gradient networks, which we will refer to as a self-gradient block and is a manifestation of a self-gradient function, and provide a convergence proof for self-gradient functions.

\subsection{Design challenges}
The goal of the proposed self-gradient block is to capture the gradient flow of the network with respect to the input and explicitly leverage this extra information within the network itself to boost defense capabilities. There are two challenges in designing such a self-gradient mechanism. First, providing the information regarding the gradient of the network can basically change the initial network gradient. As a result, the gradient that is used in the adversarial attack can be different from the gradient information which self-gradient block provides back to the network architecture.
Second, the gradient is computed from the scalar value of the loss function which includes the target label in an adversarial attack and during the test time, such information is not available to the network in order to compute the gradient.

Following we first discuss the first challenge and then provide a solution to the second challenge. 

\subsection{Gradient Convergence Theorem}
Providing an additional input to the network can potentially change the existing gradient of the DNN model. However, it is  expected that the gradient to be stationery if the added input is small. The following theorem further analyzes this problem and provides further insight how this problem can be implicitly gets resolved during the training.  

\textbf{Definition}\footnote{The notation $\nabla f(x)$ is the short form of $\nabla_x f(x)$}: The functions $f(\cdot)$ is a self-gradient function if given  the input value $x$, and the assumption that \mbox{$f^0(x)=f(x)$}: 

\begin{align}
f^n(x)=f\Big(x+\epsilon\nabla f^{n-1}(x)\Big)
\end{align}
which means the input to the function at step $n$ is the combination of the input $x$ and the gradient of the function at step~$n-1$.

\textbf{Example:}\\ $f^1(x)=f\Big(x+\epsilon\nabla f(x)\Big)$.\\ $f^2(x)=f\Big(x+\epsilon\nabla f^1(x)\Big)=f\Big(x+\epsilon\nabla f\big(x+\epsilon\nabla f(x)\big)\Big)$. 
\\
\begin{theorem}\label{th:1}  \textbf{Convergence of Self-Gradient Functions (CSGF):}
Given  the self-gradient functions $f(\cdot)$, the input value $x$ and a small noise scalar $\epsilon $ with the condition $0\leq\epsilon< 1$, then we have: \begin{align}
\lim_{n\to\infty}f^n(x)-f^{n-1}(x) = 0
\end{align}

\label{th:one}
\end{theorem}

\begin{proof}
\begin{align}
    &\lim_{n\to\infty}f^n(x)-f^{n-1}(x) \nonumber\\
    & =\lim_{n\to\infty}f(x+\epsilon\nabla f^{n-1}(x))-f(x+\epsilon\nabla f^{n-2}(x))
    \nonumber\\
    &=\lim_{n\to\infty} f(x) + \nabla f(x). \epsilon\nabla f^{n-1}(x) \nonumber \\
    &+ \epsilon\nabla f^{n-1}(x)^T H(x) \epsilon\nabla f^{n-1}(x) + ...
    \nonumber\\
    &- \lim_{n\to\infty} f(x) + \nabla f(x). \epsilon\nabla f^{n-2}(x) \nonumber \\
    &+ \epsilon\nabla f^{n-2}(x)^T H(x) \epsilon\nabla f^{n-2}(x) + ...  \nonumber\\
    &= \lim_{n\to\infty}  \epsilon
    \nabla f(x). \Big(\nabla f^{n-1}(x)-\nabla f^{n-2}(x)\Big)\nonumber\\
    &+ \lim_{n\to\infty} \epsilon^2  \Big(\nabla f^{n-1}(x)-\nabla f^{n-2}(x)\Big)^T  \cdot \nonumber\\
    & H(x) \Big(\nabla f^{n-1}(x)-\nabla f^{n-2}(x)\Big) + ...
    \nonumber\\
    &=  \lim_{n\to\infty}  \epsilon^2
    \nabla f(x). \nabla \Big(\nabla f(x).\big(\nabla f^{n-2}(x)-\nabla f^{n-3}(x)\big)\Big)\nonumber\\
    &\hspace{5cm}\vdots
    \nonumber\\
    &= \lim_{n\to\infty}  \epsilon^{n-1}
    \nabla f(x). \nabla^n \Big(\nabla f(x).\big(\nabla f^{1}(x)-\nabla f(x)\big)\Big)=0
\end{align}
\end{proof}

\begin{figure*}[!]
\setlength{\tabcolsep}{0.01cm} 
\centering
        \includegraphics[width=0.6\textwidth]{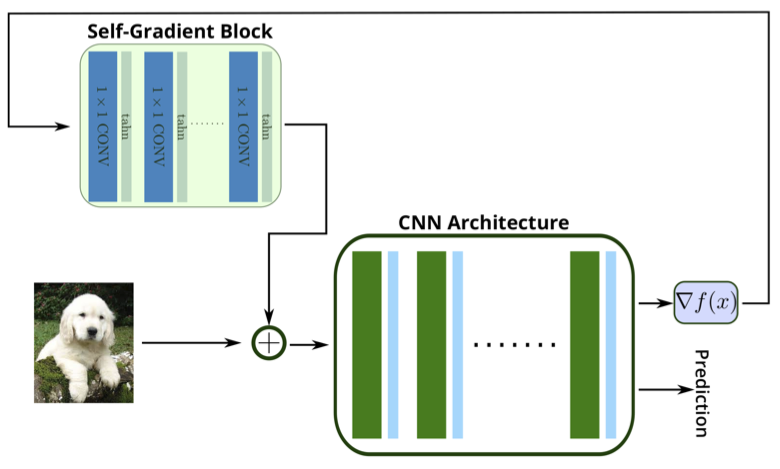}
\caption{The architecture of self-gradient network. The input data is passed through the network once and the soft-loss is calculated  based on~\eqref{eq:soft-loss}. The computed soft loss value is fed into the self-gradient block which is added to the input data when it is processed and the main network predicts the output in the second forward pass. It is worth noting that this workflow is completely different than recurrent neural networks.  }
\label{fig:arc}
\end{figure*}

Theorem~\ref{th:one} provides the first building block in constructing self-gradient networks. From this theorem, we can deduce that if the gradient information with small value of $\epsilon$ is added to the network, it is expected that the gradient of the network remains stationary in the limit. This is similar to the paradigm in designing the perturbation for adversarial example in which the perturbation has a limited norm in order to prevent the image to be unrecognizable. 
 
 Theorem~\ref{th:one} illustrates that given a self-gradient function, it is possible to provide perfect estimation with zero loss in $n$ steps of the computation. However, accessing to the ground truth information and calculating the loss and gradient in the inference time is practically impossible. As such, here we proposed the self-gradient block which can approximate the self-gradient function for a given deep neural network.

\subsection{Self-gradient Block}

Putting these pieces together, we were able to forge the self-gradient block in the architecture of the network. A self-gradient block is a processing block which aims to estimate the function gradient based on some approximations.
The scheme of the self gradient architecture is demonstrated in Figure~\ref{fig:arc}. As seen, the deep neural network requires two forward passes to calculate the output  given the input sample. On the first pass, the network computes the gradient of the input and on the second pass, it takes advantage of the gradient of the input via a $1 \times 1$ convolutional layer and $tanh$ activation and then adds it to the input. The final output of the network is determined based on this second passes. 

The main reason we used the $1 \times 1$ convolutional layer is to mimic the behaviour of the $\epsilon$ multiplied by the sign of the gradient in the $tanh$ activation function as the following layer to restrict the outputs within the range of $[-1,1]$ which is the range of the sign of the gradient. Also, the $tanh$ provides the non-linearity needed to better approximate the missing of actual loss computation in the self-gradient block. In this regard, in order to improve the results in our experiments, we stacked together $5$ layers of $1 \times 1$ convolutional layer and $tanh$ activation which considerably improves the approximation and the final results.

While the proposed self-gradient block has a simple architecture, the approximation of the loss function to be used as the input to this block is important to be effective. Following, the loss function approximation is introduced and explained in more detail.

\subsection{Soft Loss Function}

The gradient used in adversarial attack is usually taken from loss function with respect to the input. On the other hand, as this information is not available during the test time, we had to improvise an alternative scalar quantity. As such, the loss function plays an important role in the proposed self-gradient block.

Comprehensive experimental results showed that the sum of the logit of the final layer of the  deep neural network model in the image classification task can help the proposed block to learn effectively. To this end the loss function used to train the new deep neural network design is formulated as follows:
\begin{align}
    L(\theta,x) = \sum_{i=1}^c l_i (x,\theta)
    \label{eq:soft-loss}
\end{align}
where $c$ is the number of classes for the classification task and $\theta$ is the parameters of the network. $l_i$ is the score corresponding to the class $i$ computed by the layer prior to the softmax operation in the network, we name the used loss function as soft loss.

Even though this loss function is different from the loss function used in the adversarial attack, it captures most of the gradient flow of the deep neural network. Moreover, by using special architecture design, we expect self-gradient block was able to successfully plays the role of a denoising block during the adversarial attack. The key element in this design was to provide an extra information regarding the gradient flow of the deep neural network for a given input which can not be learnt in the normal training process.

\subsection{Training advantage}
Experimental results show that the self-gradient block makes the training process more affordable against some of the harshest adversarial attacks. In this regard, as it is an architectural change to the model, unlike other defense mechanism like TRADES proposed by Zhang \textit{et al.}~\cite{zhang2019theoretically} which may require exhaustive hyper parameter search, no hyper-parameter tuning for the proposed self-gradient network is necessarily needed to find the best setting against adversarial attack. 
In addition, in our extensive experiments, we found out that self gradient provides an additional significant benefit in adversarial training against multi-step attacks like PGD. In a generic adversarial training,  the model is usually trained against 10-step or higher steps against multi-step attacks like PGD in order to gain robustness against these iterative attacks. On the other hand, increasing the step number in attack approximately increases the running time by the factor of step number (i.e., usually $>10 \times$). This makes adversarial training a costly process and inefficient.
However, results show that the self-gradient block is capable of lifting this barrier and still achieves impressive results against multi-step attacks by only requiring one step adversarial attack training. More  details regarding the training process is provided in the next sections and the  supplementary materials. 

\section{Gradient Flow}
One important question may arise is that if the proposed self-gradient network architecture is differential. The gradient flow in self-gradient architecture is  similar to recurrent neural networks (RNNs) to some extents. When the  operations inside the RNN model is unfolded, the next output is directly dependent on the previous output of the model. In a similar way, in the self-gradient network architecture, the output of the network is dependent on its gradient calculated from the previous pass. 

As such, to break this dependency loop in our design, we break the forward pass in self-gradient network into two passes:
\begin{itemize}
    \item In the first pass, the input is given with gradient vector set to zero. This essentially would not add any information to the input sample in the first pass.
    \item In the second pass, the gradient of the first pass is added to the input and the the prediction output of the network is calculated in the second pass.
\end{itemize}  
As there exist no loop in this  design process, the network remains differential similar to RNN architectures.

\section{Experimental Results \& Discussion}
The proposed method is evaluated via different public datasets and compared with the state-of-the-art algorithms proposed for adversarial robustness. The competing methods are evaluated via three main adversarial attacks including FGSM~\cite{goodfellow2014explaining}, PGD~\cite{madry2017towards} and CW~\cite{carlini2017towards} attacks. 

\subsection{Experimental Setup}
As our baselines for performance comparison, we chose the Madry PGD-based adversarial training~\cite{madry2017towards} approach as well as a recently proposed adversarial training called Bilateral adversarial training~\cite{wang2019bilateral}. This method take advantage of perturbing on both  image input  and the label during training to make the model more robust against adversarial attacks. We also include the results of standard training which refers to the case `without adversarial training'. The models are tested against FGSM, PGD, and CW attacks.
\subsubsection{Datasets}
Three main datasets are used to evaluate the competing algorithms include:
\begin{itemize}
    
\item \textbf{CIFAR-10:}~\cite{krizhevsky2009learning}: is consist of $32 \times 32$ natural images categorized into 10 different class labels. The dataset has  50000 training samples and 10000 test images. 

\item \textbf{CIFAR-100}~\cite{krizhevsky2009learning}: is comprised of 50000 training   \mbox{$32 \times 32$} natural images and 10000 test images like CIFAR-10. However, the images are categorized into 100 class labels. As such, the number of training examples per class label is very limited.

\item \textbf{SVHN}~\cite{netzer2011reading}: is comprised of the \mbox{$32 \times 32$} street view house numbers images with 73257  training images and 26032 test samples.
\end{itemize}

\subsubsection{Hyperparameters}

The Wide ResNet (WRN-28-10)~\cite{zagoruyko2016wide} as the main network architecture used in conducted experiments as it is the commonly used architecture to evaluate the adversarial training algorithm~\cite{zhang2019theoretically, madry2017towards}. The factor 10 in the network name encodes how wide the network architecture is chosen. 

To evaluate the proposed self-gradient architecture, the WRN-28-10 is changed accordingly as  outlined in Section~\ref{sec:self-gradient}. The source code of our implementation will be available for interested readers\footnote{link to the code is provided in supplementary}. The model is  trained  with 200 epochs and learning rates of $\{0.1, 0.1, \text{and}~ 0.01\}$ for CIFAR-10, CIFAR-100 and SVHN respectively. These learning rate were multiplied by $0.1$ at epochs $\{50,100,150\}$. The $\epsilon$ is set to 8 for adversarial perturbation budget  similar to other works in the literature. Further details of training with extra experiments is included the supplementary.

\begin{table*}[ht] 
    	\caption{Results of the defense models accuracy on CIFAR-10. Self gradient refers to the proposed method in this paper.
    	Significant improvement is achieved with self gradient model compared with previous models with the margin of over 5 $\%$. For the multi-step attack of PGD and CW, the number besides the attack refers to the number of steps used in that attack.}
	\begin{tabular}{|l| c | c | c c c c |c c c c|} 
		\hline
		 \multirow{ 2}{*}{ Model} & Clean &FGSM &PGD10 &  PGD20  & PGD40  & PGD100 &  CW10 & CW20 & CW40 & CW100\\& &Attack & Attack & Attack & Attack & Attack & Attack & Attack & Attack & Attack \\ [0.5ex] 
		\hline\hline
		\multirow{ 2}{*}{\bf Standard}  &\multirow{ 2}{*} {95.6} &\multirow{2}{*} {36.9}&\multirow{ 2}{*} {0.0}  &\multirow{ 2}{*} {0.0} &\multirow{ 2}{*} {0.0} &\multirow{ 2}{*} {0.0} &\multirow{ 2}{*} {0.0} &\multirow{ 2}{*} {0.0} &\multirow{ 2}{*} {0.0} &\multirow{ 2}{*} {0.0} \\& &  &  &  &  &  &  &  &  &\\ 
		\hline
		\multirow{ 2}{*}{\bf Madry PGD }  &\multirow{ 2}{*} {85.7} &\multirow{2}{*} {54.9}&\multirow{ 2}{*} {45.1}  &\multirow{ 2}{*} {44.9} &\multirow{ 2}{*} {44.8} &\multirow{ 2}{*} {44.8} &\multirow{ 2}{*} {45.9} &\multirow{ 2}{*} {45.7} &\multirow{ 2}{*} {45.6} &\multirow{ 2}{*} {45.4} \\& &  &  &  &  &  &  &  &  &\\ 
		\hline
		\multirow{ 2}{*}{\bf Bilateral }  &\multirow{ 2}{*} {91.2} &\multirow{2}{*} {70.7}&\multirow{ 2}{*} {--}  &\multirow{ 2}{*} {57.5} &\multirow{ 2}{*} {--} &\multirow{ 2}{*} {55.2} &\multirow{ 2}{*} {--} &\multirow{ 2}{*} {56.2} &\multirow{ 2}{*} {--} &\multirow{ 2}{*} {53.8} \\& &  &  &  &  &  &  &  &  &\\ 
		\hline
		\multirow{ 2}{*}{\bf Self gradient }  &\multirow{ 2}{*} {90.9} &\multirow{2}{*} {81.4}&\multirow{ 2}{*} {64.37}  &\multirow{ 2}{*} {63.98} &\multirow{ 2}{*} {63.90} &\multirow{ 2}{*} {63.88} &\multirow{ 2}{*} {64.41} &\multirow{ 2}{*} {64.14} &\multirow{ 2}{*} {64.1} &\multirow{ 2}{*} {64.09 } \\& &  &  &  &  &  &  &  &  &\\ 
		\hline
	\end{tabular}

	\label{tab:cifar10}
\end{table*}

\begin{table*}
\caption{Results of the defense models accuracy on CIFAR100. Self gradient refers to the proposed method in this paper. Significant improvement is achieved with self gradient model compared with previous models with the margin of over 10 $\%$. For the multi-step attack of PGD and CW, the number besides the attack refers to the number of steps used in that attack.}
	\begin{tabular}{|l| c | c | c c c c |c c c c|} 
		\hline
		\multirow{ 2}{*}{ Model} & Clean &FGSM &PGD10 &  PGD20  & PGD40  & PGD100 &  CW10 & CW20 & CW40 & CW100\\& &Attack & Attack & Attack & Attack & Attack & Attack & Attack & Attack & Attack \\ [0.5ex] 
		\hline\hline
		\multirow{ 2}{*}{\bf Standard}  &\multirow{ 2}{*} {79.0} &\multirow{2}{*} {10.0}&\multirow{ 2}{*} {0.0}  &\multirow{ 2}{*} {0.0} &\multirow{ 2}{*} {0.0} &\multirow{ 2}{*} {0.0} &\multirow{ 2}{*} {0.0} &\multirow{ 2}{*} {0.0} &\multirow{ 2}{*} {0.0} &\multirow{ 2}{*} {0.0} \\& &  &  &  &  &  &  &  &  &\\ 
		\hline
		\multirow{ 2}{*}{\bf Madry PGD }  &\multirow{ 2}{*} {59.9} &\multirow{2}{*} {28.5}&\multirow{ 2}{*} {23.1}  &\multirow{ 2}{*} {22.6} &\multirow{ 2}{*} {22.4} &\multirow{ 2}{*} {22.3} &\multirow{ 2}{*} {24.0} &\multirow{ 2}{*} {23.2} &\multirow{ 2}{*} {23.1} &\multirow{ 2}{*} {23.0} \\& &  &  &  &  &  &  &  &  &\\ 
		\hline
		\multirow{ 2}{*}{\bf Bilateral }  &\multirow{ 2}{*} {68.2} &\multirow{2}{*} {60.8}&\multirow{ 2}{*} {--}  &\multirow{ 2}{*} {26.7} &\multirow{ 2}{*} {--} &\multirow{ 2}{*} {25.3} &\multirow{ 2}{*} {--} &\multirow{ 2}{*} {23.0} &\multirow{ 2}{*} {--} &\multirow{ 2}{*} {22.1} \\& &  &  &  &  &  &  &  &  &\\ 
		\hline
		\multirow{ 2}{*}{\bf Self gradient }  &\multirow{ 2}{*} {72.91} &\multirow{2}{*} {58.12}&\multirow{ 2}{*} {43.0}  &\multirow{ 2}{*} {42.72} &\multirow{ 2}{*} {42.67} &\multirow{ 2}{*} {42.64} &\multirow{ 2}{*} {42.74 } &\multirow{ 2}{*} {42.59} &\multirow{ 2}{*} {42.57} &\multirow{ 2}{*} {42.57 } \\& &  &  &  &  &  &  &  &  &\\ 
		\hline
	\end{tabular}
	
	\label{tab:cifar100}
\end{table*}

\begin{table*}
\caption{Results of the defense models accuracy on SVHN. Self gradient refers to the proposed method in this paper. Significant improvement is achieved with self gradient model compared with previous models with the margin of over 10 $\%$. For the multi-step attack of PGD and CW, the number besides the attack refers to the number of steps used in that attack.}
	\begin{tabular}{|l| c | c | c c c c |c c c c|} 
		\hline
		\multirow{ 2}{*}{ Model} & Clean &FGSM &PGD10 &  PGD20  & PGD40  & PGD100 &  CW10 & CW20 & CW40 & CW100\\& &Attack & Attack & Attack & Attack & Attack & Attack & Attack & Attack & Attack \\ [0.5ex] 
		\hline\hline
		\multirow{ 2}{*}{\bf Standard}  &\multirow{ 2}{*} {97.2} &\multirow{2}{*} {53.0}&\multirow{ 2}{*} {0.0}  &\multirow{ 2}{*} {0.0} &\multirow{ 2}{*} {0.0} &\multirow{ 2}{*} {0.0} &\multirow{ 2}{*} {0.0} &\multirow{ 2}{*} {0.0} &\multirow{ 2}{*} {0.0} &\multirow{ 2}{*} {0.0} \\& &  &  &  &  &  &  &  &  &\\ 
		\hline
		\multirow{ 2}{*}{\bf Madry PGD}  &\multirow{ 2}{*} {93.9} &\multirow{2}{*} {68.4}&\multirow{ 2}{*} {49.5}  &\multirow{ 2}{*} {47.9} &\multirow{ 2}{*} {46.5} &\multirow{ 2}{*} {46.0} &\multirow{ 2}{*} {49.5} &\multirow{ 2}{*} {48.7} &\multirow{ 2}{*} {48.1} &\multirow{ 2}{*} {47.3} \\& &  &  &  &  &  &  &  &  &\\ 
		\hline
		\multirow{ 2}{*}{\bf Bilateral }  &\multirow{ 2}{*} {94.1} &\multirow{2}{*} {69.8}&\multirow{ 2}{*} {--}  &\multirow{ 2}{*} {53.9} &\multirow{ 2}{*} {--} &\multirow{ 2}{*} {50.3} &\multirow{ 2}{*} {--} &\multirow{ 2}{*} {50.0} &\multirow{ 2}{*} {--} &\multirow{ 2}{*} {48.9} \\& &  &  &  &  &  &  &  &  &\\ 
		\hline
		\multirow{ 2}{*}{\bf Self gradient }  &\multirow{ 2}{*} {95.87} &\multirow{2}{*} {79.03}&\multirow{ 2}{*} {67.52}  &\multirow{ 2}{*} {66.91} &\multirow{ 2}{*} {66.65} &\multirow{ 2}{*} {66.55} &\multirow{ 2}{*} {67.16 } &\multirow{ 2}{*} {66.77} &\multirow{ 2}{*} {66.67} &\multirow{ 2}{*} {66.61 } \\& &  &  &  &  &  &  &  &  &\\ 
		\hline
	\end{tabular}
	
	\label{tab:SVHN}
\end{table*}

\subsection{Results}
Table~\ref{tab:cifar10} summarizes the results of competing  defense mechanism on the CIFAR-10 dataset. As observed, the standard model while having a high accuracy for clean dataset, it can be completely fooled via PGD and CW attacks and results zero accuracy. Madry PGD based adversarial training significantly improves the standard model performance by the margin of  $40\%$. As discussed in~\cite{zhang2019theoretically}, there is  a price that comes with robustness, and we can see a drop in the clean dataset accuracy of the Madry PGD based adversarial training. While the proposed self-gradient network and the Bilateral training approach suffer from this issue, results show that they can offer higher generalization compared to Mardy PG adversarial training. 

Experimental results demonstrates the improvement of Bilateral adversarial training approach over Madry PD adversarial training; However, the self-gradient network is able to  consistently outperforms both algorithms with over $5\%$ against  more sophisticated attacks like PGD and CW. It is also interesting to see that the improvement remains consistent even when the number of steps in attacks like PGD and CW increases. 

For the second experiment, the proposed method and competing algorithms are evaluated based on SVHN dataset. Table~\ref{tab:SVHN} further confirms the previous results reported in Table~\ref{tab:cifar10} As seen. Madry PGD based adversarial training and Bilateral training are significantly less effective than the proposed self gradient network. This table again shows that self-gradient system provides a more consistent robustness against both CW and PGD attack regardless of the step numbers in these attacks.

Table~\ref{tab:cifar100} shows  the results of different defense models on  CIFAR-100 dataset. CIFAR-100 contains 100  classes but with the same training size as CIFAR-10, dataset and this causes the models to not to be trained as effective as when trained on CIFAR-10 dataset. This is evident in the reported results in Table~\ref{tab:cifar10}.  However, a similar trend can be  observed in Table~\ref{tab:cifar100} for the competing methods. Madry PGD based adversarial training provides a significant robustness improvement over standard training. This improvement is marginally enhanced by the Bilateral training. However, Bilateral training has a slightly lower performance than the Madry PGD based adversarial training  against CW attack. The Table again shows that the proposed self-gradient network  can achieve a more consistent robustness against both CW and PGD attack regardless of the step numbers in these attacks. It also improves the results by the margin of over $10\%$.

\subsection{Impact of Self-gradient Block}
The proposed self-gradient block has shown to be effective in helping the network and improving the robustness of the model significantly. The proposed self-gradient block is being used during the training and the inference. However one important question to ask is how the network behaves if this block is disabled during the inference time. This analysis can help us to measure how effective is the proposed block even in facilitating a better training experience for the model. As such here, we study the effect of self gradient block in more detail. To measure the impact of the self gradient-block in helping the model to better learn the adversarial space, the block is disabled in the inference time to see how much it is going to impact the results.

Table~\ref{tab:ab} shows the results; to  evaluate the models  PGD10 attack  (PGD with 10 steps) is used to measure the robustness of the model and the reported result on clean type is based on the natural samples from datasets. The examined  model is trained with self-gradient block during the training phase and for different scenarios the self-gradient block is enabled and disabled. As  seen, the robustness of the models drops against adversarial attack across all datasets when self-gradient block is disabled which is expected while these drop is not significant and the model still performing well against adversarial attacks compared to state-of-the-art accuracy. On the other hand, there is no significant changes in the clean dataset accuracy in this scenario. However, in case of CIFAR-10 dataset, the accuracy over clean dataset has rose after disabling self gradient block during the test time.

These results show that after training the model with self-gradient block in the training phase, the self-gradient block can even  be disabled in the inference time to improve the efficiency with a small sacrifice in  accuracy against adversarial attacks. The results also show that the self-gradient block is particularly affecting the model accuracy against adversarial attack and does not have much impact on its accuracy for clean data sample.

\begin{table*}

\centering
\caption{Comparison of Self gradient block enabled versus the time it is disabled. Adversarial refers to PGD10 attack and clean refers to normal dataset. While the accuracy against adversarial attack is decreased for self grad disabled model, the accuracy on clean dataset is almost intact.}
	\begin{tabular}{|c| c | c  c c |} 
		\hline
		\multirow{ 2}{*}{ Model} & Dataset &CIFAR10 &CIFAR100 &  SVHN\\& Type & Dataset & Dataset & Dataset \\ [0.5ex] 
		\hline\hline
		\multirow{ 2}{*}{\bf Self gradient block disabled}  & Clean &91.15& 72.81 & 95.84 \\&Adversarial  & 60.56  & 36.65  & 66.11  \\\hline 
		\multirow{ 2}{*}{\bf Self gradient block enabled}  & Clean &90.9& 72.91 & 95.87 \\&Adversarial  & 64.37  & 43.0  & 67.52  \\ 
		\hline
	\end{tabular}
	
	\label{tab:ab}
\end{table*}

\subsection{Convergence of Self-gradient Block}

In Theorem~\ref{th:1}, we showed that self-gradient functions should converge to a stationary gradient. To evaluate this characteristic for self gradient block, we examine if its gradient  converges to a stationary gradient in the limit as described in Theorem~\ref{th:1}. 

To examine this, we took a trained network with self gradient block on CIFAR10 dataset and computed its gradient iteratively. In this regard, for the first iteration we passed the gradient vector of zero to the network and computed the self gradient block value. Then we passed this gradient value as input to the self gradient block on the next iteration and repeated this process recursively.  As  seen in Figure~\ref{fig:steps}, the difference of the norm between gradient rapidly reduces after the first step and  the gradient  ultimately converges to a value close to zero. The Figure~\ref{fig:steps} shows that the significant amount of  the gradient is absorbed in the first iteration and the consequent iterations  impose minimal impact on the gradient value, as such   in the design of self-gradient block  one loop over the network instead of multiple loops is performed. 
\begin{figure}
    \centering
    \includegraphics[width=0.37\textwidth]{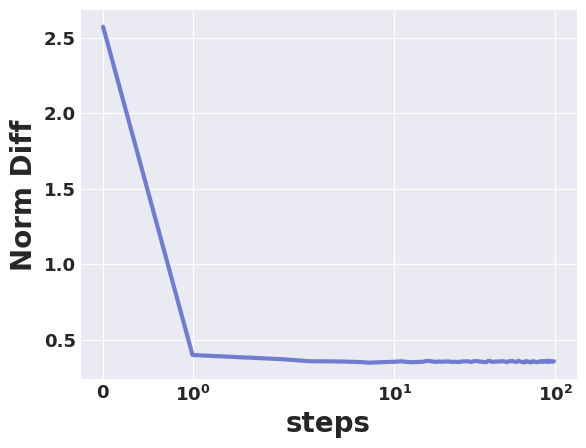}
    \caption{Convergence of gradient block to stationary gradient. Norm Diff refers to the l2 norm difference between gradient of new step from previous step.}
    \label{fig:steps}
\end{figure}

\section{Conclusion}

In this paper, we propose a novel perspective on the adversarial attack and discuss the shortcoming of the deep neural networks from a new angle. We further analyze this shortcoming from the theoretical standpoint and show  the evidence of its existence as a motivation for our novel approach to overcome this problem.  As a solution, we propose a novel architectural change to the convolutional neural network called self gradient block. 

For the first time, the self gradient block enables the neural network to learn the gradient of the output with  respect to the input. This new information enables network to learn the core component used in the adversarial attack against the network which ultimately can enhance network robustness. We also find that this block enables networks to be trained faster against multi-step adversarial attacks by requiring only one step attack during the training phase. Finally, we examine our solution and compare it with some of the established adversarial training across different benchmark datasets. Our results illustrate the effectiveness of our proposal against adversarial attacks.
{\small
\bibliographystyle{ieee_fullname}
 \bibliography{egbib}
}
\end{document}